\documentclass{article}

\usepackage{PRIMEarxiv}

\usepackage[utf8]{inputenc} 
\usepackage[T1]{fontenc}    
\usepackage{hyperref}       
\usepackage{url}            
\usepackage{booktabs}       
\usepackage{amsfonts}       
\usepackage{nicefrac}       
\usepackage{microtype}      
\usepackage{lipsum}
\usepackage{fancyhdr}       
\usepackage{graphicx}       
\graphicspath{{media/}}     

\usepackage[ruled,linesnumbered,vlined]{algorithm2e}

\usepackage{amsmath}
\usepackage{amssymb}
\usepackage{comment}
\usepackage{bbm}
\usepackage{complexity}
\usepackage{todonotes}
\usepackage{amsthm}
\usepackage{subcaption}
\usepackage{multirow}
\usepackage{makecell}

\newtheorem{example}{Example}

\newtheorem{lemma}{Lemma}

\mathchardef\mhyphen="2D

\newcommand{\argmax}{\operatorname*{\mathrm{arg\,max}}}
\newcommand{\argmin}{\operatorname*{\mathrm{arg\,min}}}

\newcommand{\topcaption}{%
\setlength{\abovecaptionskip}{6pt}%
\caption}  

\renewcommand{\d}{\mathrm{d}}
\newcommand{\Diff}[2]{\frac{\d#1}{\d#2}}

\newcommand{\Del}[2]{\frac{\partial#1}{\partial#2}}

\newcommand{\thetahat}{\hat{\theta}}

\title{Predict+Optimize for Packing and Covering LPs with Unknown Parameters in Constraints
}

\author{
  Xinyi Hu \\
  Department of Computer Science and Engineering\\
  The Chinese University of Hong Kong\\
  Hong Kong\\
  \texttt{xyhu@cse.cuhk.edu.hk} \\
   \And
  Jasper C.H. Lee \\
  Department of Computer Sciences \\
  Institute for Foundations of Data Science\\
  University of Wisconsin–Madison \\
  WI, USA\\
  \texttt{jasper.lee@wisc.edu} \\
  \And
  Jimmy H.M. Lee \\
  Department of Computer Science and Engineering\\
  The Chinese University of Hong Kong\\
  Hong Kong\\
  \texttt{jlee@cse.cuhk.edu.hk} \\
}

\begin{document}
\maketitle

\begin{abstract}
Predict+Optimize is a recently proposed framework which combines machine learning and constrained optimization, tackling optimization problems that contain parameters that are unknown at solving time.
The goal is to predict the unknown parameters and use the estimates to solve for an estimated optimal solution to the optimization problem.
However, all prior works have focused on the case where unknown parameters appear only in the optimization objective and not the constraints, for the simple reason that if the constraints were not known exactly, the estimated optimal solution might not even be feasible under the true parameters.
The contributions of this paper are two-fold.
First, we propose a novel and practically relevant framework for the Predict+Optimize setting, but with unknown parameters in both the objective and the constraints.
We introduce the notion of a correction function, and an additional penalty term in the loss function, modelling practical scenarios where an estimated optimal solution can be modified into a feasible solution after the true parameters are revealed, but at an additional cost.
Second, we propose a corresponding algorithmic approach for our framework, which handles all packing and covering linear programs.
Our approach is inspired by the prior work of Mandi and Guns, though with crucial modifications and re-derivations for our very different setting.
Experimentation demonstrates the superior empirical performance of our method over classical approaches.
\end{abstract}


\section{Introduction}


Constrained optimization problems are ubiquitous in daily life, yet, they often contain parameters that are unknown at solving time.
As an example, retail merchants wish to optimize their stocking of products in terms of revenue and cost, and yet the precise demands for each product are not known ahead of time.
The goal, then, is to 1) predict the unknown parameters and 2) solve the optimization problem using these predicted parameters, in the hopes that the estimated solution is good even under the true parameters revealed later on.
The classical approaches would learn a predictor for these unknown parameters using losses like the mean squared error, which are independent of the optimization at hand.
However, a small error for the predicted parameters in the parameter space does not necessarily guarantee a high solution quality evaluated under the true parameters.
The recent framework of Predict+Optimize by Elmachtoub and Grigas~\cite{elmachtoub2017smart,elmachtoub2022smart} proposes to instead use the more effective \emph{regret function} as the loss function, capturing the difference in objective between the estimated and true optimal solutions, both evaluated using the true parameters.

A number of prior works~\cite{wilder2019melding,elmachtoub2020decision,ali2022divide} have developed algorithmic implementations of this framework on a variety of classes of optimization problems.
Yet, all the prior works have focused on the case where only the optimization objective contains unknown parameters, and never the constraints.
This is for a simple technical reason: if we had used some predicted parameters to solve for an estimated solution, the solution might not even be feasible under the true parameters!
On the other hand, some application scenarios allow for post-hoc correction of an estimated solution into a feasible solution after the true parameters are revealed, potentially at additional cost or penalty.
Using the product stocking example again, a hard constraint is the available warehouse space, which needs to be predicted, depending on how well the already-bought products sell.
If a merchant buys in excess of the available space, they always have the option to throw away some of the newly-bought products, which would involve 1) paying a disposal company as well as 2) losing out on the profit of the thrown-away products as a ``penalty".

The contributions of this paper are two-fold.
First, we capture the above intuition and significantly generalize the Predict+Optimize framework (Section~\ref{sec:CR}), allowing us to address optimization problems with unknown parameters in both the objective and the constraints.
Specifically, we introduce the notion of a correction function, and modify the definition of regret to take into account the post-doc correction of a solution, and the associated cost and penalty.
Second, we propose an algorithmic implementation for this novel framework as applied to packing and covering linear programs (LPs), a well-studied and significant class of practically relevant optimization problems.
We give a general correction function for packing and covering LPs, and demonstrate how to learn a predictor in this setting using an approach inspired by the work of Mandi and Guns~\cite{mandi2020interior}.
We also apply our approach on 3 benchmarks to demonstrate the superior empirical performance of our method over classic learning algorithms\footnote{We allow estimated solutions to be corrected also for these classic learning algorithms, but the training itself just uses the original loss function, which is oblivious to any potential correction.}.


\section{Background}
\label{sec:background}

In this section, we describe the formulation of Predict+Optimize as it appears in prior works, on problems with unknown parameters appearing only in the objective.
The theory is stated in terms of minimization but applies of course also to maximization, upon appropriate negation.

An \emph{optimization problem} $P$ is defined as finding 
\[
    \textstyle{x^* = \argmin_x obj(x) \text{ s.t. } C(x)}
\]
where $x \in \mathbb{R}^d$ is a vector of decision variables, $obj: \mathbb{R}^d \rightarrow \mathbb{R}$ is a function mapping $x$ to a real \emph{objective value} which is to be minimized, and $C(x)$ is a set of constraints over $x$. 
We say $x^*$ is an \emph{optimal solution} and $obj(x^*)$ is the \emph{optimal value}.

In prior works, a \emph{parameterized optimization problem (Para-OP)} $P(\theta)$ extends an \emph{optimization problem} $P$ as: 
\begin{equation*}
    \textstyle{x^*(\theta) = \argmin_x  obj(x, \theta) \text{ s.t. } C(x)}
\end{equation*}
where $\theta \in \mathbb{R}^t $ is a vector of parameters.
The objective depends on $\theta$, and note that the constraints do not (in prior works).
When the parameters are known, a Para-OP is just an optimization problem.

In Predict+Optimize~\cite{elmachtoub2017smart,elmachtoub2022smart}, the true parameters $\theta \in \mathbb{R}^t$ for a Para-OP are unknown at solving time, and \emph{estimated parameters} $\hat{\theta}$ are used instead. 
Suppose that for each parameter, there are $m$ relevant features.
A learner is given $n$ observations forming a training data set $\{(A^1, \theta^1), \dots, (A^n, \theta^n) \}$, where $A^i \in \mathbb{R}^{t \times m}$ is a \emph{feature matrix} for $\theta^i$, and the task is to learn a \emph{prediction function} $f:\mathbb{R}^{t \times m} \rightarrow \mathbb{R}^t$ predicting parameters $\hat{\theta} = f(A)$ from any feature matrix $A$.

The key aspect of Predict+Optimize is to measure quality of the estimated parameters $\hat{\theta}$ using the \emph{regret function} as the loss function.
The regret is the objective difference between the \emph{true optimal solution} $x^*(\theta)$ and the \emph{estimated solution} $x^{*} (\hat{\theta})$ under the true parameters $\theta$.
Formally, the regret function $Regret(\hat{\theta}, \theta): \mathbb{R}^t \times \mathbb{R}^t \rightarrow \mathbb{R}_{\geq 0}$ is: 
\begin{equation*}
    Regret(\hat{\theta},\theta) =  obj (x^*(\hat{\theta}), \theta) - obj (x^*(\theta), \theta)
\end{equation*}
where $obj (x^*(\hat{\theta}), \theta)$ is the \emph{estimated optimal value} and $obj (x^*(\theta), \theta)$  is the \emph{true optimal value}.
Following the empirical risk minimization principle, prior learning methods~\cite{elmachtoub2020decision} aim to return the prediction function to be the function $f$ from the set of models $\mathcal{F}$ attaining the smallest average regret over the training data:
\begin{equation}
    \textstyle{f^* = \arg\min_{f \in \mathcal{F}}  \frac{1}{n}\sum^{n}_{i=1} 
    Regret(f(A^i), \theta^i)}  
    \label{eq:ERM}
\end{equation}


Mandi and Guns~\cite{mandi2020interior} proposed to use a (feedforward) neural network to predict the unknown parameters from features.
The standard approach to training neural networks is via gradient descent using the backpropagation algorithm, in order to learn the weight on each edge of the network.
Concretely, fixing a training feature matrix $A$ and a corresponding true parameter vector $\theta$, for each edge $e$ on the network with weight $w_{e}$, we need to compute the derivative $\Diff{Regret}{w_e}$.
Using the multivariate chain rule, the derivative can be decomposed as follows:
\begin{equation}
    \Diff{Regret(\hat{\theta},\theta)}{w_e} = \Del{Regret(\hat{\theta},\theta)}{x^*(\hat{\theta})} \Del{ x^*(\hat{\theta})}{ \hat{\theta}} \Del{\hat{\theta}}{w_e}
    \label{eq:oldChain}
\end{equation}
where $\Del{Regret(\hat{\theta},\theta)}{x^*(\hat{\theta})}$ is a vector with the same length as the decision variable vector $x^*$, $\Del{ x^*(\hat{\theta})}{ \hat{\theta}}$ is a matrix, and $\Del{\hat{\theta}}{w_e}$ is a vector with the same length as the number of unknown parameters.
The right hand side of the Equation~\ref{eq:oldChain} is to be interpreted as a matrix product.


On the right hand side, the first term is the gradient of the regret with respect to the estimated optimal solution.
In the context of linear programs, this is trivial to compute since the objective function is linear in $x^*$.
The third term, on the other hand, is the gradient of the estimated parameters with respect to the neural network edge weight, which can be computed efficiently using the standard backpropagation algorithm~\cite{rumelhart1986learning}.
What remains is the second term $\Del{x^*}{\thetahat}$: the derivative of each decision variable with respect to each predicted parameter.
In general, these derivatives do not exist for linear programs.
Mandi and Guns~\cite{mandi2020interior} thus proposed to use an interior-point LP solver: it generates a sequence of modified programs, with \emph{logarithmic barrier terms} of decreasing weights introduced into the objective.
Upon termination of the solver at an approximate optimum of the LP, the interior-point solver returns the approximate optimum as well as auxiliary information such as the weight of the barrier term at termination, all of which are used by Mandi and Guns to extract some gradient information related to the original problem.
Due to page limits, we do not present their approach in any detail here, but instead refer the reader directly to our new method (Section~\ref{sec:packing}), which is a substantial modification.

\section{Predict+Optimize for Unknown Constraint Parameters}
\label{sec:CR}

We now generalize the framework in the previous section to include unknown parameters also in constraints.

The notion of a Para-OP can be easily extended to allow unknown parameters in both the objective and constraints:
\begin{equation*}
    \textstyle{x^*(\theta) = \argmin_x obj(x, \theta) \text{ s.t. } C(x, \theta)}
\end{equation*}
Note that in this extension, both the objective and constraints depend on the unknown parameters $\theta$.

When constraints contain unknown parameters, the feasible region is only approximated at solving time, and the estimated solution may be infeasible under the true parameters.
Fortunately, in some applications, once the true parameters are revealed, there might be possible ways for us to correct an infeasible solution into a feasible one.
This can be formalized as a \textit{correction function}, which takes an estimated solution $x^*(\hat{\theta})$ and true parameters $\theta$ and returns a \emph{corrected solution} $x^*_{corr}(\hat{\theta},\theta)$ that is feasible under $\theta$.
The choice of correction function will be problem and application-specific; indeed, the space of correction functions depends on the situation.
The goal then is to choose a correction function that generally loses the least amount in the objective from the correction.

\begin{example}\label{example:knapsack}
Consider a simplified version of the product stocking problem.
There are 4 divisible products (e.g.~oil and rice).
Each product $i$ has a per-unit revenue $r_i$ and a per-unit weight $w_i$, and there is a maximum of $M_i$ units available for sourcing.
The goal is to make an order of $x_i$ units of item $i$, so as to maximize $\sum_{i=1}^4 r_i \cdot x_i$ subject to the constraint $\sum_{i=1}^4 w_i \cdot x_i \leq C$, where $r = [13,14,10,11]$ and $w = [5,3,4,9]$ are two arrays representing the per-unit revenues and weights of the products, as well as the constraint that $x_i \le M_i$ for all $i$.
However, the available capacity $C$ when the products arrive is unknown at solving time, depending on the volume of sales between the orders being made and the arrival of the products.
\end{example}

In Example~\ref{example:knapsack}, the products are selected based on an estimated warehouse capacity, but the prediction might be an overestimate.
One trivial correction function is to throw out the entire order, which is not useful.
A more useful correction function is to throw out some of each product to fit them into the actually available capacity.

While application scenarios may allow for post-hoc correction of an estimated solution, such correction may incur a penalty.
A \emph{penalty function} $Pen(x^*(\hat{\theta}) \to x^*_{corr}(\hat{\theta},\theta))$ takes an estimated solution $x^*(\hat{\theta})$ and the corrected solution $x^*_{corr}(\hat{\theta},\theta)$ and returns a non-negative penalty.
In Example~\ref{example:knapsack}, the correction incurs both 1) logistical costs for removing items and 2) costs of having paid for these products.

We are now ready to define the notion of \emph{post-hoc regret} $PReg(\hat{\theta},\theta)$ with respect to correction function $x^*_{corr}(\thetahat,\theta)$ and penalty function $Pen$:
\begin{equation}
    PReg(\hat{\theta},\theta) = obj(x^{*}_{corr}(\hat{\theta},\theta), \theta) - obj (x^*(\theta), \theta) + Pen(x^*(\hat{\theta}) \to x^*_{corr}(\hat{\theta},\theta))
    \label{eq:CReg_func}
\end{equation}

Given a correction function and a penalty, we will follow Mandi and Guns~\cite{mandi2020interior} and train a neural network to minimize the empirical post-hoc regret.
In the rest of the paper, we will study the application of this framework to packing and covering linear programs.
We will propose a generic correction function that should be applicable generally, and show how we can learn a neural network that performs well under the post-hoc regret.

\section{Predict+Optimize on Packing LPs}
\label{sec:packing}


In this section, we derive how we can train a neural network to predict unknown parameters in both the objective and constraints of a packing LP, under the new Predict+Optimize framework proposed in Section~\ref{sec:CR}.

Consider a packing LP in the standard form:
\begin{equation}
    x^* = \argmax_x c^\top  x \  \text{ s.t. } Gx \leq h, x \geq 0
\label{eq:packingLP}
\end{equation}
with decision variables $x \in \mathbb{R}^d$ and problem parameters $c \in \mathbb{R}^d$, $G \in \mathbb{R}^{p \times d}_{\ge 0}$, $h \in \mathbb{R}^p_{\ge 0}$.
Here, we consider the most general setting where all the problem parameters $c, G,$ and $h$ can be unknown.

We stated in Section~\ref{sec:CR} that the choice of a correction function generally depends on the specific problem and application.
On the other hand, packing LPs have a lot of structure we can exploit.
For example, the all 0s solution is always feasible.
We propose the following generic correction function, which is generally applicable for packing LPs: given an uncorrected solution $x^*$, find the largest $\lambda \in [0,1]$ such that $\lambda x^*$ satisfies the constraints under the true parameters.
This can be formalized as follows:
\begin{equation}
\begin{aligned}
&x^*_{corr}(\hat{\theta}, \theta = (c,G,h)) = \lambda x^*(\hat{\theta}) \\
\text{where }& \lambda = \max\{\lambda \in [0,1] \, | \, G(\lambda x^*(\hat{\theta})) \leq h\}\\ 
\end{aligned}
\label{eq:CF4packingLP}
\end{equation}

We also need to decide on a penalty function, which again is generally problem and application-specific.
For simplicity and for wide applicability, in the rest of the paper we will assume that the penalty function is \emph{linear}, in the sense that the penalty for the correction is the dot product between 1) the difference between the corrected and uncorrected solution vectors and 2) a vector of penalty factors.
Due to scaling reasons, we express this vector of penalty factors in units of the objective $c$, that is, the penalty vector is $\sigma \circ c$ where $\circ$ is the Hadamard/entrywise product, and $\sigma \ge 0$ is a non-negative tunable vector.
Then, the penalty function $Pen$ is formally defined as $Pen(x^*(\thetahat) \to x^*_{corr}(\thetahat,\theta)) = (\sigma \circ c)^\top (x^* - x^*_{corr})$.


With the above choices of correction and penalty, we can now write down the simplified form of post-hoc regret for packing LPs.
Note that, since packing LPs are maximization problems instead of minimization, the following has some sign differences from Equation~\ref{eq:CReg_func}.
\begin{equation}
    PReg(\hat{\theta},\theta) = c^\top(x^*(\theta) - x^{*}_{corr}(\hat{\theta},\theta)) + (\sigma \circ c)^\top  (x^*(\hat{\theta}) - x^{*}_{corr}(\hat{\theta},\theta))
\label{eq:CR4packingLP}
\end{equation}
where $\sigma \in \mathbb{R}_{\geq 0}^d$.


Following the approach of Mandi and Guns~\cite{mandi2020interior}, briefly described in Section~\ref{sec:background}, we use a neural network (of various architectures depending on the precise problem) to predict the parameters, before feeding the parameters into the interior-point LP solver of Mandi and Guns.
This interior point solver iteratively generates a sequence of relaxations to the LP, into problems of the form $$\textstyle{\argmax_x c^\top x + \mu[\sum_{i=1}^d \ln(x_i) + \sum_{i=1}^p\ln(h_i - G_i^\top x)]}$$ for a sequence of decreasing non-negative $\mu$.
Upon termination, we retrieve a solution $x$ which is approximately the optimum of the original LP, as well as the value of $\mu$ last used.

We derive how, using the solution $x$ and the barrier weight $\mu$, we can compute the relevant (approximations of) derivatives in order to train the neural network via gradient descent.
Using the law of total derivative, we get
\begin{equation}
    \Diff{PReg(\hat{\theta},\theta)}{w_e} = \left.\frac{\partial PReg(\hat{\theta},\theta)}{\partial x_{corr}^*}\right|_{x^*} \frac{\partial x_{corr}^*}{\partial x^*} \frac{\partial x^*(\hat{\theta})}{\partial \hat{\theta}} \frac{\partial \hat{\theta}}{\partial w_{e}} \nonumber
    + \left.\frac{\partial PReg(\hat{\theta},\theta)}{\partial x^*}\right|_{x^*_{corr}} \frac{\partial x^*(\hat{\theta})}{\partial \hat{\theta}} \frac{\partial \hat{\theta}}{\partial w_{e}}
\label{eq:chainRule}
\end{equation}
On the right hand side, the terms $\left.\frac{\partial PReg(\hat{\theta},\theta)}{\partial x_{corr}^*}\right|_{x^*}$ and $\left.\frac{\partial PReg(\hat{\theta},\theta)}{\partial x^*}\right|_{x^*_{corr}}$ are straightforward from (\ref{eq:CR4packingLP}): $\left.\frac{\partial PReg(\hat{\theta},\theta)}{\partial x_{corr}^*}\right|_{x^*} = -(1+\sigma)\circ \, c$ and $\left.\frac{\partial PReg(\hat{\theta},\theta)}{\partial x^*}\right|_{x^*_{corr}} = \sigma \circ \, c$.
The term $\Del{\thetahat}{w_e}$ relates only to the neural network and is handled directly by the standard backpropagation algorithm~\cite{rumelhart1986learning}.
Therefore, in the remainder of this section, we show how to compute (approximations of) $\Del{x_{corr}^*}{x^*}$ and $\Del{x^*(\thetahat)}{\thetahat}$.



\paragraph{Computing $\Del{x_{corr}^*}{x^*}$.}
The term $\Del{x_{corr}^*}{x^*}$ is determined solely by the correction function (\ref{eq:CF4packingLP}), and has nothing to do with the LP solver.
We use the law of total derivative again to decompose the term:
\begin{equation*}
    \frac{\partial x^{*}_{c}(\hat{\theta},\theta)}{\partial x^*(\hat{\theta})} = \left.\frac{\partial x^{*}_{c}(\hat{\theta},\theta)}{\partial \lambda}\right|_{x^*} \frac{\partial \lambda}{\partial x^*(\hat{\theta})} + \left.\frac{\partial x^{*}_{c}(\hat{\theta},\theta) }{\partial x^*(\hat{\theta})}\right|_\lambda
\end{equation*}
Observe that $\left.\frac{\partial x^{*}_{c}(\hat{\theta},\theta)}{\partial \lambda}\right|_{x^*}=x^*(\hat{\theta})$ and $\frac{\partial x^{*}_{c}(\hat{\theta},\theta)}{\partial x^*(\hat{\theta})}\Big|_\lambda = \lambda I$ ($I$ is an identity matrix).
It remains to derive $\frac{\partial \lambda}{\partial x^*(\hat{\theta})}$, captured in the following lemma.

\begin{lemma}
Let  $x^*(\hat{\theta})$ denote the estimated optimal solution of the packing LP shown in (\ref{eq:packingLP}), $x^*_{corr}(\thetahat,\theta) = \lambda x^*(\thetahat)$ be the correction function shown in (\ref{eq:CF4packingLP}).
Suppose that at the optimal $\lambda$ of (\ref{eq:CF4packingLP}), the $i^{th}$ inequality constraint $G_i$ is tight, namely $G_i^\top (\lambda x^*(\thetahat)) = h_i$.
Then, we have
\[
    \frac{\partial \lambda}{\partial x^*(\hat{\theta})}=- \frac{\lambda}{G_i^\top  x^*(\hat{\theta})} G_i^\top.
\]
As a corollary, we have
\begin{equation*}
    \Del{x^{*}_{corr}(\hat{\theta},\theta)}{x^*(\hat{\theta})} =  \frac{-\lambda}{G_i^\top  x^*(\hat{\theta})} x^*(\hat{\theta}) G_i^\top + \lambda I.
\end{equation*}
\end{lemma}


\begin{proof}

Since the $i^{th}$ inequality constraint $G_i$ is tight, we have:
\begin{equation}
\lambda \sum_{j=1}^n G_{ij}x^*(\hat{\theta})_j = h_i
\label{eq:coverLPCons}
\end{equation}

The implicit differentiation of Equation \ref{eq:coverLPCons} with respect to $x^*(\hat{\theta})$ is:
\begin{equation*}
    \frac{\partial}{\partial x^*(\hat{\theta})}(\lambda \sum_{j=1}^n G_{ij}x^*(\hat{\theta})_j) = \frac{\partial h_i}{\partial x^*(\hat{\theta})}
\end{equation*}
Since $x^*(\hat{\theta})$ is a vector, differentiation on the $l^{th}$ variable is:
\begin{eqnarray*}
\frac{\partial}{\partial x^*(\hat{\theta})_l}(\lambda \sum_{j=1}^n G_{ij}x^*(\hat{\theta})_j) = \frac{\partial h_i}{\partial x^*(\hat{\theta})_l}
\end{eqnarray*}
where
\begin{equation*}
    \frac{\partial}{\partial x^*(\hat{\theta})_l}(\lambda \sum_{j=1}^n G_{ij}x^*(\hat{\theta})_j)=\frac{\partial \lambda}{\partial x^*(\hat{\theta})_l}G_i^\top  x^*(\hat{\theta}) + \lambda G_{il}
\end{equation*}

Since $\frac{\partial h_i}{\partial x^*(\hat{\theta})_l}=0$, we can obtain:
\begin{equation*}
     \frac{\partial \lambda}{\partial x^*(\hat{\theta})}=- \frac{\lambda}{G_i^\top  x^*(\hat{\theta})} G_i^\top.
\end{equation*}

Since $\frac{\partial x^{*}_{corr}(\hat{\theta},\theta)}{\partial \lambda}=x^*(\hat{\theta})$, $\frac{\partial x^{*}_{corr}(\hat{\theta},\theta)}{\partial x^*(\hat{\theta})}\Big|_\lambda = \lambda I$, the gradient of the corrected optimal solution with respect to the predicted optimal solution is:
\begin{eqnarray*}
\frac{\partial x^{*}_{corr}(\hat{\theta},\theta)}{\partial x^*(\hat{\theta})} 
&=&\frac{\partial x^{*}_{corr}(\hat{\theta},\theta)}{\partial \lambda} \frac{\partial \lambda}{\partial x^*(\hat{\theta})} + \frac{\partial x^{*}_{corr}(\hat{\theta},\theta) }{\partial x^*(\hat{\theta})}\Big|_\lambda  \nonumber    \\
~&=& \frac{-\lambda}{G_i^\top  x^*(\hat{\theta})} x^*(\thetahat) G_i^\top + \lambda I .
\end{eqnarray*}

\end{proof}

\paragraph{Approximating $\Del{x^*(\thetahat)}{\thetahat}$.}
Recall that the interior point solver of Mandi and Guns solves a sequence of relaxations of the following form:
\begin{equation}
x^* = \underset{x}{\arg\max}\ c^\top  x + \mu \left[\sum_{i=1}^d \ln(x_i) +  \sum_{i=1}^p \ln(h_i - G_i^\top x)\right]
\label{eq:logbarrier_relaxation}
\end{equation}
The term $\mu [\sum_{i=1}^d \ln(x_i) +  \sum_{i=1}^p \ln(h_i - G_i^\top x)]$ is also known as a logarithmic barrier term, which is commonly used in interior-point based solving methods~\cite{boyd2004convex}.
At termination, we get the values of $x^*$ and $\mu$.
We will use these values, as well as Equation (\ref{eq:logbarrier_relaxation}), to approximate the gradient information $\Del{x^*(\thetahat)}{\thetahat}$.

In the context of the packing LP, the unknown parameter $\thetahat$ may either be $c$, $G$ or $h$.
The case of $c$ has already been derived by Mandi and Guns~\cite{mandi2020interior} (see Appendix A.1 and A.2 in their paper).
The following two lemmas captures the other two cases.




Define the notation $f(x,c,G,h) = c^\top x + \mu (\sum_{i=1}^d \ln(x_i)) + \mu (\sum_{i=1}^p \ln(h_i - G_ix))$.
Then, Problem (\ref{eq:logbarrier_relaxation}) can be expressed as finding $x^* = \argmax_x f(x,c,G,h)$.
Using this notation, we write down the following two lemmas on computing $\Del{x^*}{h}$ and $\Del{x^*}{G}$ approximately.

\begin{lemma}

Consider the LP relaxation (\ref{eq:logbarrier_relaxation}), defining $x^*$ as a function of $c, G$ and $h$.
Then, under this definition of $x^*$,
\begin{equation*}
    \Del{ x^*}{ h} = - f_{xx}(x^*)^{-1} f_{hx}(x^*)
\end{equation*}
where $f_{xx}$ denotes the matrix of second derivatives of $f$ with respect to different coordinates of $x$, and similarly for other subscripts, and explicitly:
\begin{equation*}
f_{x_{k} x_{j}}(x)=\begin{cases}
-\mu x_{j}^{-2} - \mu \sum_{i=1}^{p}G_{ij}^2/(h_i-G_i^\top x)^2 &  j=k \\
- \mu \sum_{i=1}^{p}G_{ij}G_{ik}/(h_i-G_i^\top x)^2 & j \neq k
\end{cases}
\end{equation*}
and
\begin{equation*}
    f_{h_{\ell} x_{j}}(x) = \mu G_{\ell j}/(h_\ell-G_\ell^\top x)^2
\end{equation*}
\label{lemma:PLP_h}
\end{lemma}

\begin{proof}
Since $x^* = \argmax_x f(x,c,G,h)$ is an optimum, $f_x(x^*) = \left.\frac{\partial f(x)}{\partial x}\right|_{x = x^*} = 0$.
Thus, 
\begin{equation*}
    \frac{\partial}{\partial h}f_x(x^*) = 0
\end{equation*}
By the chain rule,
\begin{equation*}
\frac{\partial}{\partial h}f_x(x^*)=f_{hx}(x^*)+ f_{xx}(x^*)\frac{\partial x^*}{\partial h}
\end{equation*}
Rearranging the aboved equation, we can obtain:
\begin{equation*}
    \frac{\partial x^*}{\partial h} = - f_{xx}(x^*)^{-1} f_{hx}(x^*)
\end{equation*}
where 
\begin{equation*}
f_{x_{k} x_{j}}(x)=\left\{\begin{array}{c}
-\mu x_{j}^{-2} - \mu \sum_{i=1}^{p}G_{ij}^2/(h_i-G_i^\top x)^2, \quad j=k \\
- \mu \sum_{i=1}^{p}G_{ij}G_{ik}/(h_i-G_i^\top x)^2, \quad j \neq k
\end{array}\right.
\end{equation*}
and
\begin{equation*}
    f_{h_{\ell} x_{j}}(x) = \mu G_{\ell j}/(h_\ell-G_\ell^\top x)^2
\end{equation*}
\end{proof}

\begin{lemma}

Consider the LP relaxation (\ref{eq:logbarrier_relaxation}), defining $x^*$ as a function of $c, G$ and $h$.
Then, under this definition of $x^*$,
\begin{equation*}
    \Del{ x^*}{ G} = - f_{xx}(x^*)^{-1} f_{Gx}(x^*)
\end{equation*}
where $f_{xx}$ denotes the matrix of second derivatives of $f$ with respect to different coordinates of $x$, and similarly for other subscripts, and explicitly:
\begin{equation*}
f_{x_{k} x_{j}}(x)=\begin{cases}
-\mu x_{j}^{-2} - \mu \sum_{i=1}^{p}G_{ij}^2/(h_i-G_i^\top x)^2 & j=k \\
- \mu \sum_{i=1}^{p}G_{ij}G_{ik}/(h_i-G_i^\top x)^2 & j \neq k
\end{cases}
\end{equation*}
and
\begin{equation*}
f_{G_{\ell q} x_{j}}(x)=\begin{cases}
-\mu G_{\ell j} x_{q}/(h_\ell-G_\ell^\top x)^2-\mu/(h_\ell-G_\ell^\top x) & q=j \\
-\mu G_{\ell j} x_{q}/(h_\ell-G_\ell^\top x)^2 & q \neq j.
\end{cases}
\end{equation*}
\label{lemma:PLP_G}
\end{lemma}

\begin{proof}
Since $x^* = \argmax_x f(x,c,G,h)$ is an optimum, $f_x(x^*) = \left.\frac{\partial f(x)}{\partial x}\right|_{x = x^*} = 0$.
Thus, 
\begin{equation*}
    \frac{\partial}{\partial G}f_x(x^*) = 0
\end{equation*}
By the chain rule,
\begin{equation*}
\frac{\partial}{\partial G}f_x(x^*)=f_{Gx}(x^*)+ f_{xx}(x^*)\frac{\partial x^*}{\partial G}
\end{equation*}
Rearranging the aboved equation, we can obtain:
\begin{equation*}
    \frac{\partial x^*}{\partial G} = - f_{xx}(x^*)^{-1} f_{Gx}(x^*)
\end{equation*}
where 
\begin{equation*}
f_{x_{k} x_{j}}(x)=\left\{\begin{array}{c}
-\mu x_{j}^{-2} - \mu \sum_{i=1}^{p}G_{ij}^2/(h_i-G_i^\top x)^2, \quad j=k \\
- \mu \sum_{i=1}^{p}G_{ij}G_{ik}/(h_i-G_i^\top x)^2, \quad j \neq k
\end{array}\right.
\end{equation*}
and
\begin{equation*}
f_{G_{\ell q} x_{j}}(x)=\left\{\begin{array}{c}
-\mu G_{\ell j} x_{q}/(h_\ell-G_\ell^\top x)^2-\mu/(h_\ell-G_\ell^\top x), \quad q=j \\
-\mu G_{\ell j} x_{q}/(h_\ell-G_\ell^\top x)^2, \quad q \neq j.
\end{array}\right.
\end{equation*}
\end{proof}


We end this section with a remark that the LP solver of Mandi and Guns~\cite{mandi2020interior} in fact returns more information than just $x$ and $\mu$.
In their work, they start not with (\ref{eq:logbarrier_relaxation}) but with the homogeneous self-dual (HSD) formulation of the original LP, involving the extra information returned by the solver, and perform derivative calculations similar in spirit to our lemmas in this section.
However, in our context of unknown $G$ and $h$, if we also tried using the HSD formulation for gradient calculations, we would end up with derivatives that are degenerate.
For this reason, we have opted to use the simpler Equation (\ref{eq:logbarrier_relaxation}) which, as we demonstrate in the experiments in Section~\ref{sec:experiments}, appears to work well in practice.

\section{Predict+Optimize on Covering LPs}
\label{sec:covering}


Covering LPs are closely related to packing LPs---in fact, they are the duals of each other.
Consider a covering LP in standard form:
\begin{equation}
    x^* = \underset{x}{\arg\min}\ c^\top  x \  \text{ s.t. } Gx \geq h, x \geq 0
\label{eq:coveringLP}
\end{equation}
with decision variables $x \in \mathbb{R}^d$ and problem parameters $c \in \mathbb{R}^d$, $G \in \mathbb{R}^{p \times d}$, $h \in \mathbb{R}^p$. 
We are again in the general setting where all the problem parameters $c, G,$ and $h$ can be unknown.

Performing Predict+Optimize on covering LPs is essentially the same as in the previous section, up to some sign changes to account for changed inequality directions and minimization vs maximization.
The only non-trivial difference is the need to change the correction function.
Instead of scaling down an uncorrected solution for feasibility, we will scale up in covering LPs, defined formally as follows:
\begin{equation}
\begin{aligned}
&x^*_{corr}(\hat{\theta}, \theta = (c,G,h)) = \lambda x^*(\hat{\theta}) \\
\text{where } & \lambda = \min\{\lambda \ge 1 \, | \, G(\lambda x^*(\hat{\theta})) \ge h\}\\ 
\end{aligned}
\label{eq:CF4coveringLP}
\end{equation}
We use the same penalty function as in the packing LP case.
The differentiation calculations from the last section apply essentially verbatim to covering LPs apart from minor sign differences.

\begin{lemma}
Let  $x^*(\hat{\theta})$ denote the estimated optimal solution of the covering LP shown in (9), $x^*_{corr}(\thetahat,\theta) = \lambda x^*(\thetahat)$ be the correction function shown in (10).
Suppose that at the optimal $\lambda$ of (10), the $i^{th}$ inequality constraint $G_i$ is tight, namely $G_i^\top (\lambda x^*(\thetahat)) = h_i$.
Then, we have
\[
    \frac{\partial \lambda}{\partial x^*(\hat{\theta})}=- \frac{\lambda}{G_i^\top  x^*(\hat{\theta})} G_i^\top.
\]
As a corollary, we have
\begin{equation*}
    \Del{x^{*}_{corr}(\hat{\theta},\theta)}{x^*(\hat{\theta})} =  \frac{-\lambda}{G_i^\top  x^*(\hat{\theta})} x^*(\hat{\theta}) G_i^\top + \lambda I.
\end{equation*}
\end{lemma}

\begin{proof}

Since the $i^{th}$ inequality constraint $G_i$ is tight, we have:
\begin{equation}
\lambda \sum_{j=1}^n G_{ij}x^*(\hat{\theta})_j = h_i
\label{eq:coverLPCons}
\end{equation}

The implicit differentiation of Equation \ref{eq:coverLPCons} with respect to $x^*(\hat{\theta})$ is:
\begin{equation*}
    \frac{\partial}{\partial x^*(\hat{\theta})}(\lambda \sum_{j=1}^n G_{ij}x^*(\hat{\theta})_j) = \frac{\partial h_i}{\partial x^*(\hat{\theta})}
\end{equation*}
Since $x^*(\hat{\theta})$ is a vector, differentiation on the $l^{th}$ variable is:
\begin{eqnarray*}
\frac{\partial}{\partial x^*(\hat{\theta})_l}(\lambda \sum_{j=1}^n G_{ij}x^*(\hat{\theta})_j) = \frac{\partial h_i}{\partial x^*(\hat{\theta})_l}
\end{eqnarray*}
where
\begin{equation*}
    \frac{\partial}{\partial x^*(\hat{\theta})_l}(\lambda \sum_{j=1}^n G_{ij}x^*(\hat{\theta})_j)=\frac{\partial \lambda}{\partial x^*(\hat{\theta})_l}G_i^\top  x^*(\hat{\theta}) + \lambda G_{il}
\end{equation*}

Since $\frac{\partial h_i}{\partial x^*(\hat{\theta})_l}=0$, we can obtain:
\begin{equation*}
     \frac{\partial \lambda}{\partial x^*(\hat{\theta})}=- \frac{\lambda}{G_i^\top  x^*(\hat{\theta})} G_i^\top.
\end{equation*}

Since $\frac{\partial x^{*}_{corr}(\hat{\theta},\theta)}{\partial \lambda}=x^*(\hat{\theta})$, $\frac{\partial x^{*}_{corr}(\hat{\theta},\theta)}{\partial x^*(\hat{\theta})}\Big|_\lambda = \lambda I$, the gradient of the corrected optimal solution with respect to the predicted optimal solution is:
\begin{eqnarray*}
\frac{\partial x^{*}_{corr}(\hat{\theta},\theta)}{\partial x^*(\hat{\theta})} 
&=&\frac{\partial x^{*}_{corr}(\hat{\theta},\theta)}{\partial \lambda} \frac{\partial \lambda}{\partial x^*(\hat{\theta})} + \frac{\partial x^{*}_{corr}(\hat{\theta},\theta) }{\partial x^*(\hat{\theta})}\Big|_\lambda  \nonumber    \\
~&=& -\frac{\lambda}{G_i^\top  x^*(\hat{\theta})} x^*(\thetahat) G_i^\top + \lambda I .
\end{eqnarray*}

\end{proof}

\begin{lemma}
In the context of covering LP, consider the LP relaxation in the following form:
\begin{equation}
x^* = \underset{x}{\arg\min}\ c^\top  x - \mu \left[\sum_{i=1}^d \ln(x_i) -  \sum_{i=1}^p \ln(G_i^\top x - h_i)\right]
\label{eq:logbarrier_relaxation}
\end{equation}
Defining $x^*$ as a function of $c, G$ and $h$.
Then, under this definition of $x^*$,
\begin{equation*}
    \Del{ x^*}{ h} = - f_{xx}(x^*)^{-1} f_{hx}(x^*)
\end{equation*}
where $f_{xx}$ denotes the matrix of second derivatives of $f$ with respect to different coordinates of $x$, and similarly for other subscripts, and explicitly:
\begin{equation*}
f_{x_{k} x_{j}}(x)=\left\{\begin{array}{c}
\mu x_{j}^{-2} + \mu \sum_{i=1}^{p}G_{ij}^2/(h_i-G_i^\top x)^2, \quad j=k \\
 \mu \sum_{i=1}^{p}G_{ij}G_{ik}/(h_i-G_i^\top x)^2, \quad j \neq k
\end{array}\right.
\end{equation*}
and
\begin{equation*}
    f_{h_{\ell} x_{j}}(x) = - \mu G_{\ell j}/(h_\ell-G_\ell^\top x)^2
\end{equation*}
\label{lemma:PLP_h}
\end{lemma}

\begin{proof}
Since $x^* = \argmin_x f(x,c,G,h)$ is an optimum, $f_x(x^*) = \left.\frac{\partial f(x)}{\partial x}\right|_{x = x^*} = 0$.
Thus, 
\begin{equation*}
    \frac{\partial}{\partial h}f_x(x^*) = 0
\end{equation*}
By the chain rule,
\begin{equation*}
\frac{\partial}{\partial h}f_x(x^*)=f_{hx}(x^*)+ f_{xx}(x^*)\frac{\partial x^*}{\partial h}
\end{equation*}
Rearranging the aboved equation, we can obtain:
\begin{equation*}
    \frac{\partial x^*}{\partial h} = - f_{xx}(x^*)^{-1} f_{hx}(x^*)
\end{equation*}
where 
\begin{equation*}
f_{x_{k} x_{j}}(x)=\left\{\begin{array}{c}
\mu x_{j}^{-2} + \mu \sum_{i=1}^{p}G_{ij}^2/(h_i-G_i^\top x)^2, \quad j=k \\
 \mu \sum_{i=1}^{p}G_{ij}G_{ik}/(h_i-G_i^\top x)^2, \quad j \neq k
\end{array}\right.
\end{equation*}
and
\begin{equation*}
    f_{h_{\ell} x_{j}}(x) = - \mu G_{\ell j}/(h_\ell-G_\ell^\top x)^2
\end{equation*}
\end{proof}


\begin{lemma}

In the context of covering LP, consider the LP relaxation in the following form:
\begin{equation}
x^* = \underset{x}{\arg\min}\ c^\top  x - \mu \left[\sum_{i=1}^d \ln(x_i) -  \sum_{i=1}^p \ln(G_i^\top x - h_i)\right]
\label{eq:logbarrier_relaxation}
\end{equation}
Defining $x^*$ as a function of $c, G$ and $h$.
Then, under this definition of $x^*$,
\begin{equation*}
    \Del{ x^*}{ G} = - f_{xx}(x^*)^{-1} f_{Gx}(x^*)
\end{equation*}
where $f_{xx}$ denotes the matrix of second derivatives of $f$ with respect to different coordinates of $x$, and similarly for other subscripts, and explicitly:
\begin{equation*}
f_{x_{k} x_{j}}(x)=\left\{\begin{array}{c}
\mu x_{j}^{-2} + \mu \sum_{i=1}^{p}G_{ij}^2/(h_i-G_i^\top x)^2, \quad j=k \\
\mu \sum_{i=1}^{p}G_{ij}G_{ik}/(h_i-G_i^\top x)^2, \quad j \neq k
\end{array}\right.
\end{equation*}
and
\begin{equation*}
f_{G_{\ell q} x_{j}}(x)=\left\{\begin{array}{c}
\mu G_{\ell j} x_{q}/(h_\ell-G_\ell^\top x)^2+\mu/(h_\ell-G_\ell^\top x), \quad q=j \\
\mu G_{\ell j} x_{q}/(h_\ell-G_\ell^\top x)^2, \quad q \neq j.
\end{array}\right.
\end{equation*}
\label{lemma:PLP_G}
\end{lemma}

\begin{proof}
Since $x^* = \argmax_x f(x,c,G,h)$ is an optimum, $f_x(x^*) = \left.\frac{\partial f(x)}{\partial x}\right|_{x = x^*} = 0$.
Thus, 
\begin{equation*}
    \frac{\partial}{\partial G}f_x(x^*) = 0
\end{equation*}
By the chain rule,
\begin{equation*}
\frac{\partial}{\partial G}f_x(x^*)=f_{Gx}(x^*)+ f_{xx}(x^*)\frac{\partial x^*}{\partial G}
\end{equation*}
Rearranging the aboved equation, we can obtain:
\begin{equation*}
    \frac{\partial x^*}{\partial G} = - f_{xx}(x^*)^{-1} f_{Gx}(x^*)
\end{equation*}
where 
\begin{equation*}
f_{x_{k} x_{j}}(x)=\left\{\begin{array}{c}
\mu x_{j}^{-2} + \mu \sum_{i=1}^{p}G_{ij}^2/(h_i-G_i^\top x)^2, \quad j=k \\
 \mu \sum_{i=1}^{p}G_{ij}G_{ik}/(h_i-G_i^\top x)^2, \quad j \neq k
\end{array}\right.
\end{equation*}
and
\begin{equation*}
f_{G_{\ell q} x_{j}}(x)=\left\{\begin{array}{c}
\mu G_{\ell j} x_{q}/(h_\ell-G_\ell^\top x)^2 + \mu/(h_\ell-G_\ell^\top x), \quad q=j \\
\mu G_{\ell j} x_{q}/(h_\ell-G_\ell^\top x)^2, \quad q \neq j.
\end{array}\right.
\end{equation*}
\end{proof}

\section{Experimental Evaluation}
\label{sec:experiments}

We evaluate the proposed method on 3 benchmarks: a maximum flow transportation problem with unknown edge capacities, an alloy production problem with unknown chemical composition in the raw materials, and a fractional knapsack problem with unknown rewards and weights.
We compare our method with $5$ classical regression methods~\cite{friedman2001elements} including ridge regression (Ridge), $k$-nearest neighbors ($k$-NN), classification and regression tree (CART), random forest (RF), and neural network (NN).
All of these methods train the prediction models with their classic loss function.
We also apply the chosen correction function of each problem to the estimated solutions for these classical regression methods, in order to ensure feasibility, to compute the post-hoc regret.
However, the correction function has nothing to do with the training of these classic methods.
The methods of $k$-NN, RF and NN as well as our method have hyperparameters, which we tune via cross-validation:
for $k$-NN, we tried $k \in \lbrace 1,3,5 \rbrace$;
for RF, we try different numbers of trees in the forest $\{10, 50, 100 \}$;
for both NN and our method, we treat the learning rate, epochs and weight decay as hyperparameters.
The final hyperparameter choices are shown in Table \ref{table:Hyperparameters}. 

\begin{table}[]
\centering
\resizebox{0.7\textwidth}{!}{
\begin{tabular}{|l||c|c|c|}
\hline
\multirow{2}{*}{Model} & \multicolumn{3}{c|}{Hyperaprameters}  \\
\cline{2-4}
                       & Max flow transportation                                      & Alloy production                                             & Fractional knapsack                                          \\
                       \hline
Proposed               & \makecell{optimizer: optim.Adam;\\ learning rate: $5^{-6}$;\\ $\mu=10^{-3}$; epochs=8} & \makecell{optimizer: optim.Adam;\\ learning rate: $5^{-6}$;\\ $\mu=10^{-3}$; epochs=8} & \makecell{optimizer: optim.Adam;\\ learning rate: $5^{-6}$;\\ $\mu=10^{-3}$; epochs=8} \\\hline
$k$-NN                    & \multicolumn{3}{c|}{k=5}                                                                                                                                                                    \\\hline
RF                     & \multicolumn{3}{c|}{n\_estimator=100}                                                                                                                                                       \\\hline
NN                     & \makecell{optimizer: optim.Adam;\\ learning rate: $10^{-3}$;\\ epochs=8}    & \makecell{optimizer: optim.Adam;\\ learning rate: $10^{-3}$;\\ epochs=8}    & \makecell{optimizer: optim.Adam;\\ learning rate: $10^{-3}$;\\ epochs=8}  \\
\hline
\end{tabular}}
\setlength{\abovecaptionskip}{6pt}%
\topcaption{Hyperparameters of the maximum flow transportation, alloy production, and fractional knapsack problems.}
\label{table:Hyperparameters}
\end{table}

Ridge, $k$-NN, CART and RF are implemented using \textit{scikit-learn}~\cite{scikit-learn}. 
The neural network is implemented using \textit{PyTorch}~\cite{NEURIPS2019_9015}.
All models are trained with Intel(R) Xeon(R) CPU @ 2.20GHz.
To compute the optimal solution of an LP under the true parameters, we use the LP solver from \textit{OR-Tools}~\cite{ortools} instead of the solver of Mandi and Guns.



\paragraph{A maximum flow transportation problem with unknown capacities.}
In our first experiment, we formulate a transportation problem as a single-source-single-sink maximum flow problem (MFP).
To formulate it as a packing LP, we use the formulation where the decision variables each correspond to a simple path from the source to the sink.
In this experiment, the unknown parameters are the edge capacities, which is the $h$ vector in the packing LP.
We experiment in a setting where the goal is to use Predict+Optimize to learn which paths we will be using for transport, and proportionally how much flow we will be sending along each path---for example, the prediction is used to apply for permits from a city council for sending a lot of traffic along particular routes.
Given that we are less concerned about predicting the actual flow magnitudes, in this experiment we set the penalty factor $\sigma$ to the all-0s vector.

We conduct experiments on $3$ real-life graphs: POLSKA~\cite{SNDlib10} with 12 vertices and 18 edges, USANet~\cite{lucerna2009efficiency} with 24 vertices and 43 edges, and G\'{E}ANT~\cite{WinNT} with 40 vertices and 61 edges.
Given that we are unable to find datasets specifically for this max-flow problem, we follow the experimental approach of Demirovic et al.~\cite{demirovic2019investigation, demirovic2019predict, demirovic2020dynamic} and use real data from a different problem (the ICON scheduling competition) as numerical values required for our experiment instances.
In this dataset, each unknown edge capacity has $8$ features.
For experiments on POLSKA and USANet, out of the available 789 instances, 610 are used for training and 179 for testing the model performance,
while for experiments on G\'{E}ANT, out of the available 620 instances, 490 are used for training and 130 for testing the model performance.

For both NN and our method, we use a 3-layer fully-connected network with 16 neurons per hidden layer.


\begin{table*}[]
\centering
\resizebox{0.8\textwidth}{!}{
\begin{tabular}{l||c|c|c|c|c|c||c}
\hline
\multicolumn{1}{c||}{PReg} & Proposed          & Ridge                  & $k$-NN          & CART           & RF            & NN & TOV \\
\hline
POLSKA & \textbf{10.00±0.67}          & 11.20±0.73                & 14.39±0.83               & 16.65±1.06               & 12.30±0.90             & 12.18±1.08             & 88.66±1.10              \\
USANet & \textbf{16.64±1.34}          & 19.52±1.16                & 22.89±1.58               & 24.15±1.51               & 22.27±1.34             & 18.62±1.23             & 96.22±1.38              \\
G\'{E}ANT  & \textbf{10.84±1.10}          & 12.47±1.14                & 15.13±1.08               & 17.01±1.59               & 12.52±1.19             & 12.05±1.13             & 98.71±1.98    \\  \hline
\end{tabular}}
\topcaption{Mean post-hoc regrets and standard deviations for the maximum flow transportation problem.}
\label{table:MFP_PR}
\end{table*}

\begin{table*}[]
\centering
\resizebox{0.9\textwidth}{!}{
\begin{tabular}{l||c|c|c|c|c|c}
\hline
\multicolumn{1}{c||}{MSE} & Proposed          & Ridge                  & $k$-NN          & CART           & RF            & NN                   \\
\hline
POLSKA                  & 1.45E+04±2.63E+04 & \textbf{290.75±127.31} & 363.13±120.51 & 474.00±145.07  & 309.94±123.44 & 324.38±132.49 \\
USANet                  & 1.76E+04±2.20E+04 & \textbf{755.54±90.39}  & 913.79±91.48  & 1626.40±195.31 & 779.04±83.86  & 903.86±105.96 \\
G\'{E}ANT                   & 1.62E+04±2.58E+04 & \textbf{700.35±72.66}  & 842.45±75.78  & 1484.84±203.11 & 704.96±76.64  & 828.18±95.18    \\  \hline
\end{tabular}}
\topcaption{Mean square errors and standard deviations for the maximum flow transportation problem.}
\label{table:MFP_MSE}
\end{table*}

\begin{figure}[]
\centering
\includegraphics[width=0.6\textwidth]{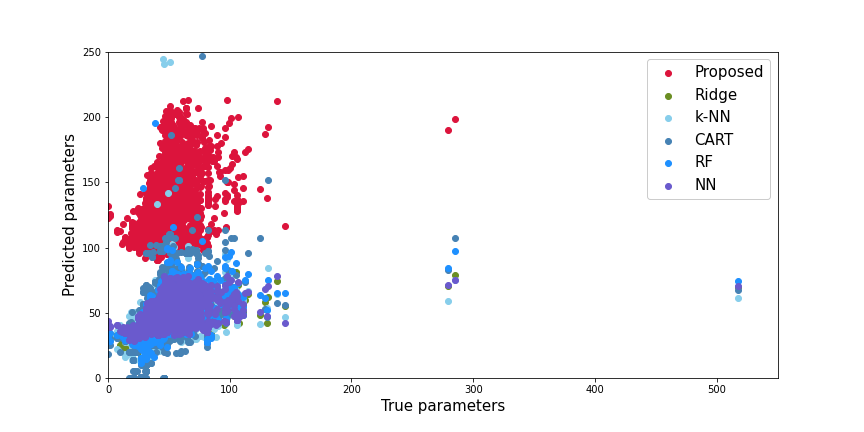}
\topcaption{True parameters vs Predicted parameters.}
\label{fig:groundTrueVSpredictions}
\end{figure}

Tables \ref{table:MFP_PR} and \ref{table:MFP_MSE} report the mean post-hoc regrets and standard deviations across 10 runs, and the mean square errors (MSE) and standard deviations across 10 runs for each approach on the maximum flow transportation problem with unknown capacities respectively.

 
Table \ref{table:MFP_PR} shows that the proposed method achieves the best performance in all cases. 
Compared with other methods, the proposed method obtains at least 10.71\% smaller post-hoc regret on POLSKA, 10.67\% smaller on USANet, and 10.02\% smaller on G\'{E}ANT.
We also report the average True Optimal Values (TOV) in the last column of Table \ref{table:MFP_PR} for reference.
The proposed method achieves 11.49\% relative error on POLSKA, 16.23\% relative error on USANet, and 10.28\% relative error on G\'{E}ANT.

For comparison, we also show a table of the mean squared error (MSE), i.e.~squared $\ell_2$ error, of the predicted parameters, across different methods and graphs, in Table \ref{table:MFP_MSE}.
Even though the goal is to minimize post-hoc regret and it is unreasonable to evaluate our method on the MSE, we present these results anyway for all our experiments, as they help illustrate the behavior of our method.
In this experiment, the MSE of the proposed method is drastically worse than all the other methods, while ridge regression achieves the best performance in all of the cases unsurprisingly since it is explicitly designed to learn in $\ell_2$ error, and RF always achieves the second best performance.
We argue that this is by design: our method optimizes for learning in terms of post-hoc regret, while all the other classical methods learn to minimize in MSE.
There does remain the question of why our method has that bad of an MSE.
Here, we give a scatterplot (Figure~\ref{fig:groundTrueVSpredictions}) of the norm of the predicted parameters versus the true parameters, across all the methods.
As we can see in Figure \ref{fig:groundTrueVSpredictions}, the predicted parameters values of the proposed method are several orders of magnitude higher than the true parameters values, i.e., our method predicting the unknown parameters at several orders of magnitude larger than the true parameters.
The reason for this phenomenon lies in the problem formulation, where we are trying to predict which paths to send flows through, and are less concerned with predicting the precise amount of flow.
As a modelling choice, therefore, we picked the penalty factor $\sigma = 0$.
Note that, since the unknown parameters are the ``$h$" vector in the packing LP, if we scale up the $h$ vector, then the corresponding solutions $x$ are scaled up by the same factor.
Thus, the phenomenon is equivalent to the estimated solution being much larger than the true optimal solution.
This is fine from the Predict+Optimize perspective: the correction function scales down an over-capacity estimated solution, and so the predictor only needs to predict the direction of the solution vector; even if it gives a far-too-large norm, the correction function will fix the magnitude at no cost.
Our learning algorithm appears to have learnt to exploit this correction function, and nonetheless, the estimated solution still gives the desired information---which paths to send flow along in the graph.

If we did care about predicting the actual flow values, then we would set the penalty factor to a non-zero value.
In the next couple of experiments, we explore how the penalty factor affects the performance of our method, in terms of both the post-hoc regret and the MSE of the predicted parameters.
We note again that the penalty factor is a property of the application, and not an algorithmic choice we make.


\paragraph{An alloy production problem with unknown chemical compositions in raw materials.}
In our second experiment, we consider an alloy production problem that is expressible as a covering LP.
An alloy production plant needs to produce a certain amount of a particular alloy, requiring a mixture of $M$ kinds of metals. 
To that end, it must acquire at least $req_m$ tons of each of the $m \in [M]$ metals. 
The raw materials are to be obtained from $K$ suppliers, each supplying a different type of ore.
The ore supplied by site $k \in [K]$ contains a $con_{km} \in [0,1]$ fraction of material $m$ at a price of $cost_k$ per ton.
The objective is to meet the minimum material requirements for each metal, at the minimum cost.
The decision variables $x_k$ are the number of tons of ores to order from each site $k$.
Affected by the uncertainty in the mining process, the metal concentration (\% of the $m \in M$ material per ton) of each ore is unknown, i.e.~$con_{km}$ is unknown, which is the $G$ matrix in the covering LP.

Following the correction function and penalty described in Sections~\ref{sec:covering} and~\ref{sec:packing} respectively, if the estimated solution does not meet the minimum tonnage requirements of any metal, the alloy production plant will scale up its order by a factor of $\lambda \ge 1$ (from Equation~\ref{eq:CF4coveringLP}) across all the suppliers.
On the other hand, for this after-the-fact order, each supplier $k$ will charge a new cost of $(1+\sigma_k)cost_k$ per ton of its ore, instead of the previous cost of $cost_k$.
We experimented on various values of penalty factors $\sigma_k$ and we will report and discuss how the value of $\sigma_k$ affects the performance of the prediction pipeline.
We stress again that the value of $\sigma_k$ is from the application, and not an algorithmic choice.

We conduct experiments on two real alloys: brass and an alloy blend for strengthening Titanium.
For brass, $2$ kinds of metal materials, Cu and Zn, are required~\cite{kabir2010study}, that is $M=2$.
The requirements of the two materials are, proportionally, $req=[627.54, 369.72]$.
For the titanium-strengthening alloy, $4$ kinds of metal materials, C, Al, V, and Fe, are required~\cite{kahraman2005joining}, i.e., $M=4$.
The requirements of the four materials are $req=[0.8, 60, 40, 2.5]$.
Since we could not find any real data on the concentration of metals in ores, we again use real data from a different problem (a knapsack problem \cite{paulus2021comboptnet}) as numerical values in our experiment instances.
In this dataset, each unknown metal concentration is related to 4096 features.
For experiments on both of the two alloys productions, 350 instances are used for training and 150 instances for testing the model performance.

For NN and our method, we use a 5-layer fully connected network with 512 neurons per hidden layer.

We conduct experiments on $5$ types of penalty factor ($\sigma$) settings:
the all-0s vector, and then 4 vectors where each entry is i.i.d.~uniformly sampled from $[0.25\pm 0.015], [0.5 \pm 0.015], [1.0 \pm 0.015]$, and $[2.0\pm 0.015]$ respectively.
This random sampling of $\sigma$ ensures that the penalty factor for each supplier is different, but remain roughly in the same scale.

\begin{table*}[]
\centering
\resizebox{0.95\textwidth}{!}{
\begin{tabular}{l|l||c|c|c|c|c|c||c}
\hline
\multicolumn{2}{c||}{PReg}                               & \multirow{2}{*}{Proposed} & \multirow{2}{*}{Ridge}       & \multirow{2}{*}{$k$-NN}        & \multirow{2}{*}{CART}         & \multirow{2}{*}{RF}          & \multirow{2}{*}{NN}    & \multirow{2}{*}{TOV}               \\ \cline{1-2}
Alloy                  & Penalty factor    &                           &                        &                       &                       &                     &                     &                              \\
\hline
\multirow{5}{*}{Brass}    & 0                                  & \textbf{37.66±4.52}   & 61.93±3.17   & 65.68±5.76   & 87.57±8.83   & 61.40±2.96   & 61.46±6.69   & \multirow{5}{*}{312.02±6.94} \\
                          & 0.25±0.015                         & \textbf{68.16±6.26}   & 75.16±4.48   & 80.11±7.85   & 109.94±10.04 & 74.11±4.14   & 73.93±6.07   &                              \\
                          & 0.5±0.015                          & \textbf{82.91±5.45}   & 88.36±6.24   & 94.52±10.19  & 132.24±11.59 & 86.77±5.81   & 86.36±6.16   &                              \\
                          & 1±0.015                            & \textbf{107.64±6.85}  & 114.80±10.30 & 123.37±15.08 & 176.91±15.55 & 112.16±9.69  & 111.25±8.31  &                              \\
                          & 2±0.015                            & \textbf{150.47±12.99} & 167.64±18.69 & 181.05±25.29 & 266.19±24.29 & 162.91±17.65 & 161.03±15.46 &                              \\
\hline
\multirow{5}{*}{Titanium-alloy} & 0                                  & \textbf{4.07±0.75}    & 6.15±0.67    & 6.51±0.50    & 7.95±0.64    & 5.93±0.63    & 5.87±0.66    & \multirow{5}{*}{30.27±0.54}  \\
                          & 0.25±0.015                         & \textbf{6.45±0.81}    & 7.54±0.81    & 8.03±0.59    & 10.05±0.67   & 7.22±0.75    & 7.14±0.79    &                              \\
                          & 0.5±0.015                          & \textbf{7.90±0.561}    & 8.92±0.96    & 9.56±0.69    & 12.15±0.73   & 8.53±0.88    & 8.52±0.90    &                              \\
                          & 1±0.015                            & \textbf{10.73±0.81}   & 11.69±1.28   & 12.59±0.92   & 16.34±0.87   & 11.12±1.16   & 11.08±1.19   &                              \\
                          & 2±0.015                            & \textbf{14.17±1.31}   & 17.23±1.92   & 18.69±1.41   & 24.72±1.24   & 16.32±1.75   & 16.25±1.72   &         \\                    
\hline                          
\end{tabular}}
\topcaption{Mean post-hoc regrets and standard deviations for the alloy production problem.}
\label{table:delivery_PR}
\end{table*}
\begin{table*}[]
\centering
\resizebox{0.85\textwidth}{!}{
\begin{tabular}{l|l||c|c|c|c|c|c}
\hline
\multicolumn{2}{c||}{MSE}                               & \multirow{2}{*}{Proposed} & \multirow{2}{*}{Ridge}       & \multirow{2}{*}{$k$-NN}        & \multirow{2}{*}{CART}         & \multirow{2}{*}{RF}          & \multirow{2}{*}{NN}                   \\ \cline{1-2}
\multicolumn{1}{l|}{Alloy}                  & Penalty factor    &                           &                              &                              &                               &                              &                                       \\ \hline
\multirow{5}{*}{Brass}    & 0                                  & 395.81±331.56       & \multirow{5}{*}{39.33±0.64} & \multirow{5}{*}{43.68±0.92} & \multirow{5}{*}{73.98±1.74} & \multirow{5}{*}{\textbf{37.43±0.40}} & \multirow{5}{*}{37.80±0.47}          \\
                          & 0.25±0.015                         & 168.27±38.07        &                             &                             &                             &                                      &                                      \\
                          & 0.5±0.015                          & 37.33±0.58          &                             &                             &                             &                                      &                                      \\
                          & 1±0.015                            & \textbf{36.97±0.56} &                             &                             &                             &                                      &                                      \\
                          & 2±0.015                            & 38.22±2.37          &                             &                             &                             &                                      &                                      \\
\hline
\multirow{5}{*}{Titanium-alloy} & 0                                  & 301.41±213.73       & \multirow{5}{*}{38.93±0.32} & \multirow{5}{*}{43.92±0.53} & \multirow{5}{*}{73.82±0.47} & \multirow{5}{*}{37.51±0.33}          & \multirow{5}{*}{\textbf{36.60±0.26}} \\
                          & 0.25±0.015                         & 48.23±7.95          &                             &                             &                             &                                      &                                      \\
                          & 0.5±0.015                          & 44.69±5.74          &                             &                             &                             &                                      &                                      \\
                          & 1±0.015                            & 39.00±2.63          &                             &                             &                             &                                      &                                      \\
                          & 2±0.015                            & 45.28±4.28          &                             &                             &                             &                                      &   \\
\hline
\end{tabular}}
\topcaption{Mean square errors and standard deviations for the alloy production problem.}
\label{table:delivery_MSE}
\end{table*}

Tables \ref{table:delivery_PR} and \ref{table:delivery_MSE} report the mean post-hoc regrets and standard deviations across 10 runs, and the mean square errors (MSE) and standard deviations across 10 runs for each approach on the alloy production problem with unknown metal concentrations, across the different scales of penalty factor $\sigma$.

When the penalty factor is 0, our method improves the solution quality substantially, obtaining at least 38.67\% smaller post-hoc regret than the other methods in brass production, and at least 30.73\% smaller post-hoc regret in titanium-alloy production.
When the penalty factor is non-zero as given in the last paragraph, our method obtains at least 7.80\%, 3.99\%, 3.24\%, and 6.56\% smaller post-hoc regret respectively in brass production, 
and at least 9.65\%, 7.30\%, 3.14\%, and 12.82\% smaller post-hoc regret respectively in titanium-alloy production.
The results suggest that the advantages of the proposed method on solution quality first decreases and then increases as the penalty factor $\sigma$ grows.
The average True Optimal Values (TOV) are reported in the last column of Table \ref{table:delivery_PR}.
The relative errors of all the methods grow larger when the penalty factor grows larger. 
For example, the relative errors of the proposed method are 11.77\%, 20.77\%, 26.57\%, 34.34\%, and 47.03\% on brass production when the penalty factors are all zero, or are sampled from $[0.25\pm 0.015], [0.5 \pm 0.015], [1.0 \pm 0.015]$, $[2.0\pm 0.015]$ respectively.

We show also the MSE of the predicted parameters in Table \ref{table:delivery_MSE}, across different methods.
As discussed in the previous max-flow experiment, when the penalty is 0, the MSE for our method can be very large as the lack of penalty gets exploited by the method.
Interestingly, as we observe in Table \ref{table:delivery_MSE}, the MSE for our method first decreases and then increases as $\sigma$ grows (the growth at the end is slightly difficult to read on this plot).
Here we explain why the MSE values of the proposed method may increase when the penalty term grows too large.
When the penalty is non-zero but somewhat small, it acts as a regularizer to prevent our method from exploiting the correction function as in the previous experiment.
On the other hand, as the penalty increases, the post-hoc regret becomes dominated by the penalty term.
As such, our method is strongly disincentivized to use \emph{any} correction whatsoever.
Therefore, when $\sigma$ is large, our method tends to be conservative and always predicts parameters that make the estimated solution a bit too large.
This explains why the MSE of the predicted parameters gets bigger again (albeit not by much) as $\sigma$ increases.


\paragraph{Fractional knapsack problem with unknown prices and weights.}
The last experiment is on the fractional knapsack problem with unknown rewards and weights.
The unknown parameters appear in both objective ``$c$" and constraints ``$G$" of the packing LP.
In our setting, word descriptions of a collection of $M$ infinitely-divisible items is presented to the algorithm, from which the weight $w_i$ and reward $c_i$ of each item $i$ need to be predicted. 
The player’s goal is to maximize the total reward of (fractionally) selected items without exceeding a known fixed capacity of the knapsack.
We use the dataset of Paulus et al.~\cite{paulus2021comboptnet}, in which each fractional knapsack instance consists of $10$ items and each item has $4096$ features  related to its reward and weight.

For both NN and our method, we use a 5-layer fully-connected network with 512 neurons per hidden layer.

In line with the choice of correction function and penalty in Section~\ref{sec:packing}, if the estimated solution violates the capacity constraint, items will need to be removed at a penalty and in a proportional manner (i.e.~the over-capacity knapsack is scaled down).
If the change in the amount of item $i$ is $\Delta_i$, then the penalty for this removal is $\sigma_i c_i \Delta_i$.

We conduct experiments on $4$ different capacities: 50, 100, 150, and 200.
We use 700 instances for training and 300 instances for testing the model performance.
Identically to the second experiment, we use $5$ scales of penalty factors: all-0s penalty, and penalty factor $\sigma$ with i.i.d.~entries drawn uniformly from $[0.25 \pm 0.015], [0.5 \pm 0.015], [1.0\pm 0.015]$, and $[2.0 \pm 0.015]$. 

\begin{table*}[]
\centering
\resizebox{0.9\textwidth}{!}{
\begin{tabular}{l|l||c|c|c|c|c|c||c}
\hline
\multicolumn{2}{c||}{PReg}                               & \multirow{2}{*}{Proposed} & \multirow{2}{*}{Ridge}       & \multirow{2}{*}{$k$-NN}        & \multirow{2}{*}{CART}         & \multirow{2}{*}{RF}          & \multirow{2}{*}{NN}    & \multirow{2}{*}{TOV}               \\ \cline{1-2}
Capacity                  & Penalty factor    &                           &                        &                       &                       &                     &                     &                              \\
\hline
\multirow{5}{*}{50}  & 0          & \textbf{35.36±0.51}       & 38.00±0.89             & 36.95±1.04            & \textbf{35.53±0.71}   & 37.90±0.65          & 39.75±1.18          & \multirow{5}{*}{90.79±0.46}  \\
                     & 0.25±0.015 & \textbf{38.17±0.76}       & 39.17±0.86             & \textbf{38.46±0.96}   & 38.85±0.75            & 38.87±0.58          & 40.51±1.03          &                              \\
                     & 0.5±0.015  & \textbf{39.57±0.85}       & 40.33±0.83             & 39.97±0.90            & 42.16±0.82            & \textbf{39.85±0.53} & 41.26±0.90          &                              \\
                     & 1.0±0.015  & \textbf{41.90±0.85}       & 42.65±0.82             & 42.99±0.84            & 48.80±1.04            & \textbf{41.99±0.47} & 42.77±0.71          &                              \\
                     & 2.0±0.015  & \textbf{44.92±0.91}       & 47.30±0.90             & 49.03±1.00            & 62.08±1.63            & 45.71±0.63          & 45.79±0.86          &                              \\
\hline
\multirow{5}{*}{100} & 0          & \textbf{45.66±0.66}       & 49.52±1.29             & 48.20±1.31            & 48.08±0.75            & 49.85±1.31          & 52.19±1.84          & \multirow{5}{*}{156.46±0.79} \\
                     & 0.25±0.015 & \textbf{49.97±0.86}       & 51.12±1.22             & \textbf{50.38±1.14}   & 51.88±0.71            & 51.19±1.23          & 53.25±1.56          &                              \\
                     & 0.5±0.015  & \textbf{52.27±0.66}       & 52.73±1.17             & \textbf{52.36±1.08}   & 55.66±0.75            & 52.53±1.15          & 54.31±1.32          &                              \\
                     & 1.0±0.015  & \textbf{55.71±1.12}       & 55.93±1.15             & 56.23±0.98            & 63.25±1.01            & \textbf{55.74±0.63} & 56.44±1.05          &                              \\
                     & 2.0±0.015  & \textbf{58.88±0.79}       & 62.35±1.36             & 64.25±0.97            & 78.42±1.82            & 60.57±0.93          & 60.69±1.66          &                              \\
\hline
\multirow{5}{*}{150} & 0          & \textbf{42.01±0.37}       & 47.56±1.08             & 46.16±1.13            & 46.91±0.67            & 48.09±0.97          & 49.78±2.02          & \multirow{5}{*}{207.92±0.99} \\
                     & 0.25±0.015 & \textbf{46.59±0.23}       & 49.37±1.02             & 48.37±1.04            & 50.49±0.66            & 49.68±0.87          & 51.08±1.58          &                              \\
                     & 0.5±0.015  & \textbf{50.25±0.59}       & 51.20±0.98             & \textbf{50.58±0.97}   & 54.07±0.74            & 51.27±0.79          & 52.38±1.19          &                              \\
                     & 1.0±0.015  & \textbf{54.07±0.66}       & 54.83±1.01             & 54.99±0.95            & 61.23±1.07            & 54.44±0.69          & \textbf{54.97±0.86} &                              \\
                     & 2.0±0.015  & \textbf{58.40±0.63}       & 62.11±1.38             & 63.81±1.31            & 75.55±1.96            & 60.78±0.84          & 60.54±2.15          &                              \\
\hline
\multirow{5}{*}{200} & 0          & \textbf{25.70±0.36}       & 33.07±0.98             & 32.73±0.92            & 33.18±0.88            & 33.63±0.84          & 34.67±2.13          & \multirow{5}{*}{246.86±1.20} \\
                     & 0.25±0.015 & \textbf{31.50±0.50}       & 34.91±0.92             & 34.91±0.89            & 36.36±0.83            & 35.33±0.80          & 36.19±1.55          &                              \\
                     & 0.5±0.015  & \textbf{35.08±0.69}       & 36.76±0.90             & 37.10±0.91            & 39.55±0.89            & 37.03±0.81          & 37.71±1.09          &                              \\
                     & 1.0±0.015  & \textbf{39.54±0.45}       & 40.45±0.98             & 41.47±1.06            & 45.92±1.22            & 40.42±0.92          & 40.76±1.20          &                              \\
                     & 2.0±0.015  & \textbf{44.59±0.55}       & 47.83±1.44             & 50.22±1.66            & 58.65±2.22            & 47.20±1.39          & 46.85±3.58          &                              \\
\hline
\end{tabular}}
\topcaption{Mean post-hoc regrets and standard deviations for the fractional knapsack problem.}
\label{table:knapsack_PR}
\end{table*}
\begin{table*}[]
\centering
\resizebox{0.85\textwidth}{!}{
\begin{tabular}{ll||c|c|c|c|c|c}
\hline
\multicolumn{2}{c||}{MSE}                               & \multirow{2}{*}{Proposed} & \multirow{2}{*}{Ridge}       & \multirow{2}{*}{$k$-NN}        & \multirow{2}{*}{CART}         & \multirow{2}{*}{RF}          & \multirow{2}{*}{NN}                   \\ \cline{1-2}
\multicolumn{1}{l|}{Capacity}                  & Penalty factor    &                           &                              &                              &                               &                              &                                       \\ \hline
\multicolumn{1}{l|}{\multirow{5}{*}{50}}  & 0          & 190.05±3.89               & \multirow{20}{*}{75.40±0.65} & \multirow{20}{*}{83.47±0.77} & \multirow{20}{*}{140.51±1.75} & \multirow{20}{*}{72.04±0.73} & \multirow{20}{*}{\textbf{71.66±0.58}} \\  
\multicolumn{1}{l|}{}                     & 0.25±0.015 & 79.75±2.33                &                              &                              &                               &                              &                                       \\  
\multicolumn{1}{l|}{}                     & 0.5±0.015  & 79.52±2.69                &                              &                              &                               &                              &                                       \\  
\multicolumn{1}{l|}{}                     & 1.0±0.015  & 75.30±1.20                &                              &                              &                               &                              &                                       \\  
\multicolumn{1}{l|}{}                     & 2.0±0.015  & 72.88±1.31                &                              &                              &                               &                              &                                       \\ \cline{1-3}
\multicolumn{1}{l|}{\multirow{5}{*}{100}} & 0          & 162.95±22.35              &                              &                              &                               &                              &                                       \\  
\multicolumn{1}{l|}{}                     & 0.25±0.015 & 83.28±3.32                &                              &                              &                               &                              &                                       \\  
\multicolumn{1}{l|}{}                     & 0.5±0.015  & 77.57±1.62                &                              &                              &                               &                              &                                       \\  
\multicolumn{1}{l|}{}                     & 1.0±0.015  & 72.28±1.08                &                              &                              &                               &                              &                                       \\  
\multicolumn{1}{l|}{}                     & 2.0±0.015  & \textbf{71.57±0.58}       &                              &                              &                               &                              &                                       \\ \cline{1-3}
\multicolumn{1}{l|}{\multirow{5}{*}{150}} & 0          & 161.92±22.60              &                              &                              &                               &                              &                                       \\  
\multicolumn{1}{l|}{}                     & 0.25±0.015 & 80.33±2.05                &                              &                              &                               &                              &                                       \\  
\multicolumn{1}{l|}{}                     & 0.5±0.015  & 80.96±1.98                &                              &                              &                               &                              &                                       \\  
\multicolumn{1}{l|}{}                     & 1.0±0.015  & 75.39±1.15                &                              &                              &                               &                              &                                       \\  
\multicolumn{1}{l|}{}                     & 2.0±0.015  & \textbf{71.35±0.66}       &                              &                              &                               &                              &                                       \\ \cline{1-3}
\multicolumn{1}{l|}{\multirow{5}{*}{200}} & 0          & 151.91±35.18              &                              &                              &                               &                              &                                       \\  
\multicolumn{1}{l|}{}                     & 0.25±0.015 & 79.12±3.26                &                              &                              &                               &                              &                                       \\  
\multicolumn{1}{l|}{}                     & 0.5±0.015  & 75.40±1.50                &                              &                              &                               &                              &                                       \\  
\multicolumn{1}{l|}{}                     & 1.0±0.015  & \textbf{70.62±1.03}       &                              &                              &                               &                              &                                       \\  
\multicolumn{1}{l|}{}                     & 2.0±0.015  & \textbf{71.58±0.66}       &                              &                              &                               &                              &                                       \\ \hline
\end{tabular}}
\topcaption{Mean square errors and standard deviations for the fractional knapsack problem.}
\label{table:knapsack_MSE}
\end{table*}

Table \ref{table:knapsack_PR} shows the post-hoc regrets of the different methods across the different scales of penalty factors.
The performance of the proposed method is at least as good as other classical approaches when the capacity is 50, 100, or 150, and is consistent better than others when the capacity is 200.
Observing a similar trend as in the alloy production experiment, the improvements of our method over other classical methods, in terms of the post-hoc regret, first decreases and then increases as the penalty factor $\sigma$ grows. 
The relative errors of all the methods grow smaller when the capacity grows larger, for example, the relative errors of the proposed method are around 38-49\%, 29-37\%, 20-28\%, 10-18\% when the capacity is 50, 100, 150, and 200 respectively.

As in the previous experiments, we also compare the MSE of the parameters predicted by our method against the other methods, as shown in Table \ref{table:knapsack_MSE}.
Similar to the other experiments, when the penalty term is zero, the predicted parameters of the proposed method are shifted by several orders of magnitude from the true parameters (the post-hoc regret is small but the MSE value is large).
Then, as $\sigma$ grows, the MSE of our method decreases to roughly the same as the other methods, before growing slightly again as $\sigma$ becomes large and the predictor learnt from our method becomes conservative.


\paragraph{Runtime Analysis}
Table \ref{table:runtime} shows the average runtime across 10 simulations for different optimization problems.
In the alloy production problem and the fractional knapsack problem, the runtimes of the proposed method are comparable to NN, and are much better than RF.
In the maximum flow transportation problem, the runtimes of the proposed method are comparable to NN in POLSKA and G\'{E}ANT, but the runtime of the proposed method is large in USANet.
The reason is that we use the formulation where the decision variables each correspond to a simple path from the source to the sink.
Thus, when the number of paths is large (the number of paths in USANet is 242), the number of the decision variables of the LP is large and the LP requires more time to be solved.

\begin{table*}[h]
\centering
\resizebox{0.9\textwidth}{!}{
\begin{tabular}{|l|ccc|cc|cccc|}
\hline
{ }                             & \multicolumn{3}{c|}{Maximum   flow transportation}                                   & \multicolumn{2}{c|}{Alloy   production}      & \multicolumn{4}{c|}{Fractional   knapsack}                                                                              \\ \cline{2-10} 
\multirow{-2}{*}{{ Runtime(s)}} & \multicolumn{1}{c|}{POLSKA}       & \multicolumn{1}{c|}{USANet}       & G\'{E}ANT        & \multicolumn{1}{c|}{Brass}  & Titanium-alloy & \multicolumn{1}{c|}{Capacity=50} & \multicolumn{1}{c|}{Capacity=100} & \multicolumn{1}{c|}{Capacity=150} & Capacity=200 \\ \hline
{ Proposed}                     & \multicolumn{1}{c|}{18.65}        & \multicolumn{1}{c|}{132.22}       & 15.48        & \multicolumn{1}{c|}{228.00} & 331.38         & \multicolumn{1}{c|}{131.49}      & \multicolumn{1}{c|}{132.89}       & \multicolumn{1}{c|}{139.44}       & 132.37       \\ \hline
{ Ridge}                        & \multicolumn{1}{c|}{\textless{}1} & \multicolumn{1}{c|}{\textless{}1} & \textless{}1 & \multicolumn{1}{c|}{20.22}  & 56.89          & \multicolumn{4}{c|}{22.33}                                                                                              \\ \hline
{ $k$-NN}                          & \multicolumn{1}{c|}{\textless{}1} & \multicolumn{1}{c|}{\textless{}1} & \textless{}1 & \multicolumn{1}{c|}{25.14}  & 70.22          & \multicolumn{4}{c|}{26.00}                                                                                              \\ \hline
CART                                                & \multicolumn{1}{c|}{\textless{}1} & \multicolumn{1}{c|}{\textless{}1} & \textless{}1 & \multicolumn{1}{c|}{30.33}  & 94.89          & \multicolumn{4}{c|}{34.83}                                                                                              \\ \hline
RF                                                  & \multicolumn{1}{c|}{4.11}         & \multicolumn{1}{c|}{11.00}        & 11.89        & \multicolumn{1}{c|}{959.50} & 2552.25        & \multicolumn{4}{c|}{1034.07}                                                                                            \\ \hline
NN                                                  & \multicolumn{1}{c|}{10.33}        & \multicolumn{1}{c|}{12.82}        & 13.89        & \multicolumn{1}{c|}{212.22} & 321.11         & \multicolumn{4}{c|}{135.80}                                                                                             \\ \hline
\end{tabular}}
\topcaption{Average runtime (in seconds) for the maximum flow transportation, alloy production, and fractional knapsack problems.}
\label{table:runtime}
\end{table*}

\section{Summary}

We proposed the first Predict+Optimize framework addressing the scenario where the constraints may contain unknown parameters.
Specifically, we introduced the novel notions of correction function and post-hoc regret into the framework.
Algorithmically, we focused on packing and covering linear programs---a large and widely-studied class of problems---and presented a method to train parameter predictors in our novel framework.
Empirical results in 3 benchmarks demonstrate better prediction performance of our method over 5 classical methods which do not take the correction function into account during training.


\bibliographystyle{unsrt}  
\bibliography{references}

\end{document}